%% file: ___preprint___Search_Is_a_Sampling_Policy__Causal_Correction_of_Evidence_Selection_Bias_in_Deep_Research_Agents.tex
\documentclass{article} 
\usepackage{xspace}

\newcommand{\method}[1]{\textnormal{#1}}
\newcommand{\methodname}[1]{\textnormal{#1}}
\usepackage{iclr2027_conference,times}

\input{math_commands.tex}

\usepackage{graphicx}
\usepackage{url}
\usepackage{amsmath}
\usepackage{amssymb}
\usepackage{amsthm}
\usepackage{booktabs}
\usepackage{array}
\usepackage{xcolor}
\definecolor{citecolor}{HTML}{0071bc}
\usepackage[pagebackref=false,breaklinks=true,letterpaper=true,colorlinks,citecolor=citecolor,bookmarks=false]{hyperref}

\newtheorem{theorem}{Theorem} \newtheorem{assumption}{Assumption}

\graphicspath{{figures/}}

\title{Search Shapes Conclusions: Auditing Evidence Selection Bias in Deep Research Agents}

\author{
Shuyao Xiao$^{1,2}$, Shengling Wang$^{1}$, Xuan Chen$^{2}$,
Ke Chao$^{1}$, Ming Cui$^{2}$, Feifei Qian$^{1}$,\\
Chaoyang Mei$^{2}$, Fanlin Meng$^{2}$, Lulu Wang$^{1}$,
Ziming Yu$^{1}$, Junxi Yin$^{2}$ \\
$^{1}$Beijing Normal University \quad
$^{2}$Ke Holdings \\
\texttt{xiaoshuyao@mail.bnu.edu.cn}
}

\iclrfinalcopy 
\begin{document}

\maketitle
\lhead{Preprint}
\begin{abstract}
Deep Research agents synthesize evidence into cited reports, yet a well-cited report can still reach a misleading conclusion. Citation correctness checks whether cited sources support individual claims. It does not show whether adaptive search exposed a representative view of all documents made available for evaluation, which we call the candidate pool. Early findings redirect later queries, document choices, and stopping, so the documents an agent reads form a selective sample. Existing evaluations rarely account for this selection. We formulate the problem as adaptive evidence sampling and introduce Causal Evidence Selection Correction (CESS). CESS predicts each candidate document's evidence direction and corrects the candidate-pool average using the logged probabilities of selecting each document and reaching each search round. Shrinkage stabilizes short searches, while intervals replace point estimates when some documents cannot be sampled. We also prove that estimating the average evidence direction of a common pool differs from measuring how a change in search policy alters the evidence read. The latter requires intervention. On  questions from the MS2 systematic-review benchmark, CESS reduces mean absolute error against the candidate-pool average by $9.2\%$ and reduces the estimate's change under opposing document rankings by $39.4\%$ relative to averaging the evidence scores of documents read. Across trajectories from a public Open Deep Research agent, the corresponding reductions reach $60.1\%$ and $87.2\%$. A further 4,800 trajectories under paired interventions confirm that correcting a pool estimate and measuring a policy effect are different tasks. CESS therefore audits whether the evidence direction underlying a report reflects the documents available for evaluation, while a separate intervention analysis measures the effect of search decisions. 
\end{abstract}

\section{Introduction}
\label{sec:introduction}

Deep Research agents use large language models (LLMs) to generate queries, retrieve evidence, and synthesize cited reports for open-ended questions~\citep{du2026deepresearchbench}. Their reliability depends not only on whether each citation supports its local claim, but also on whether the search exposed evidence that could change the overall conclusion. Existing evaluations assess citation support~\citep{gao2023alce}, factual correctness~\citep{min2023factscore}, and report utility~\citep{du2026deepresearchbench}. All three may appear satisfactory even when the agent has assembled a one-sided evidence set.

This risk arises from adaptive evidence acquisition. An early supporting document can lead the agent to issue similar queries, open more supporting results, and stop before finding counterevidence. The report may faithfully summarize the documents read while differing from the direction supported by the available evidence. Report inspection alone cannot determine whether this shift arose from exposure, document selection, or stopping, making evidence selection a distinct reliability problem.

Recent work documents reports whose supported citations nevertheless conflict with broader scientific evidence~\citep{huang-etal-2026-deepfact} and shows how early search directions can reinforce themselves~\citep{zhou2026hyposearch}. These studies improve evidence checking or collection. We address the complementary measurement problem. From the subset an agent opened, how can we estimate the direction supported by a specified candidate pool? How can we separately measure the effect of changing the search policy? We use \emph{evidential conclusion} to mean a scalar summary of whether the documents support one side of a focal claim or the other. It does not refer to the linguistic quality of the report. The \emph{candidate pool} is the prespecified set of documents available for evaluating a question. We call the average evidence direction over this pool the \emph{candidate-pool target}. The corresponding average over only the documents read is the \emph{Opened Mean}. To measure dependence on document order, we define \emph{ranking sensitivity} as the absolute change in an estimate when the same pool is ordered supporting-first rather than opposing-first.

Three features make this measurement difficult. First, earlier findings shape later choices, so opened documents are not a representative sample of the candidate pool. Second, stopping removes later opportunities to observe evidence. Third, a policy may assign zero probability to part of the pool, precluding point identification from the observed trajectory. These are forms of selection bias under adaptive data collection~\citep{russo2016controlling,meng2018statistical}, compounded by sequential queries and stopping decisions.

We address the first question with Causal Evidence Selection Correction (CESS). Given a candidate pool and an aggregation rule, CESS predicts an evidence score for every candidate and corrects the resulting pool prediction using the documents read and their logged probabilities of selection and of reaching each round. When these probabilities are logged and every document can be selected, the raw estimator recovers the candidate-pool target in expectation. Shrinkage stabilizes short trajectories, while an interval replaces the point estimate when part of the pool cannot be sampled.

The second question concerns attribution. If two searches over the same pool are both corrected to the same target, their corrected estimates should agree. Their difference therefore cannot quantify how much the policies changed the evidence opened. We formalize this incompatibility and measure policy effects with paired interventions on selection and stopping. Figure~\ref{fig:overview} shows the search process and the two audit questions. The contributions are as follows.

\begin{figure*}[t]
    \centering
    \includegraphics[width=0.9\textwidth]{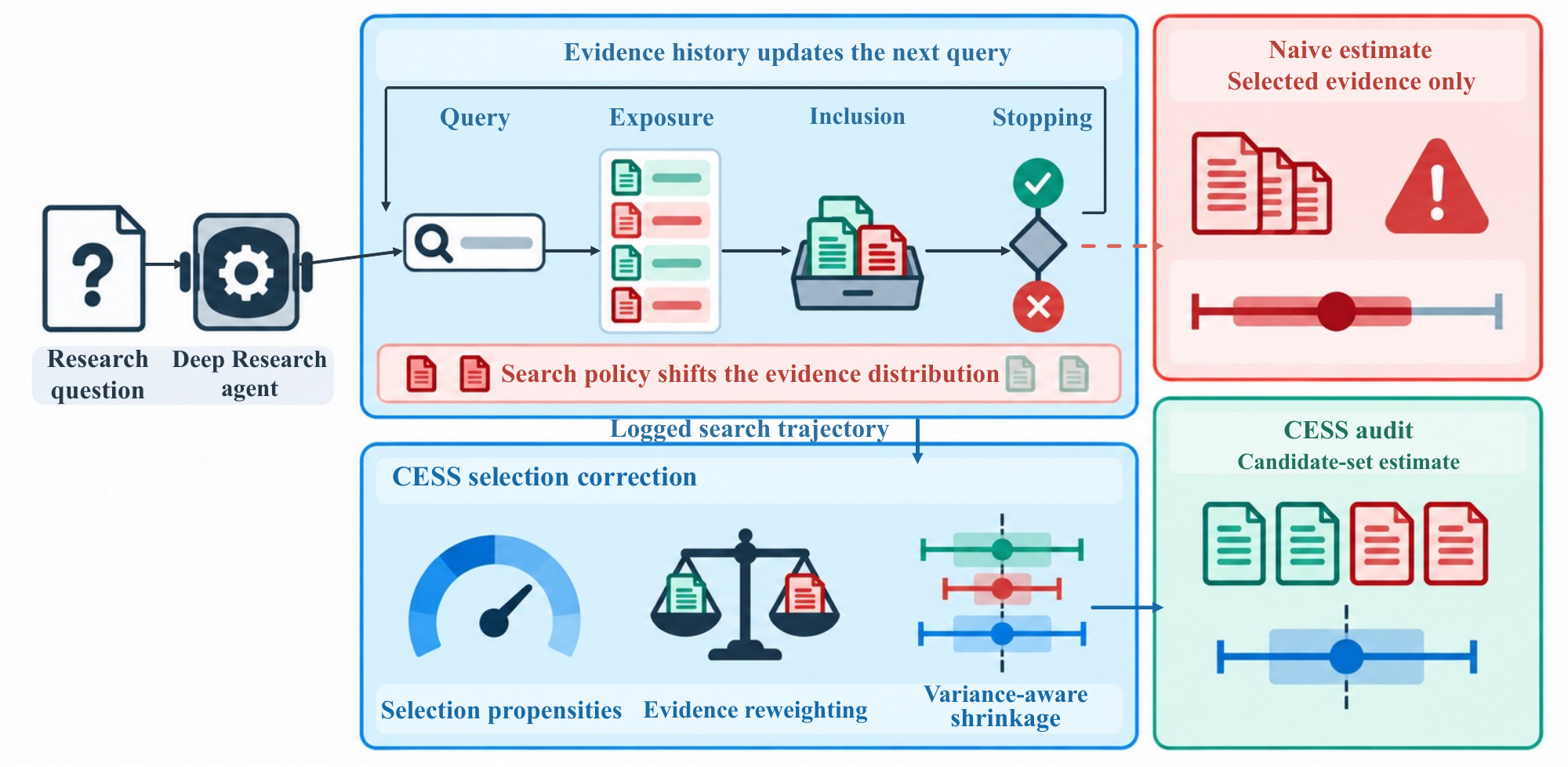}
\caption{How earlier evidence shapes later queries, document selection (``Inclusion'' in the diagram), and stopping. The naive estimate averages the evidence scores of documents read. This is the Opened Mean defined in the text. CESS estimates the candidate-pool target from the logged trajectory, while paired interventions measure how changing search decisions changes the average evidence direction among the documents read.}
    \label{fig:overview}
\end{figure*}

\begin{enumerate}
\item \textbf{A defined target for auditing evidence selection.} We separate the evidential conclusion supported by a candidate pool from the conclusion formed from opened documents, and prove that an estimate designed to recover the same pool target under every search policy cannot also identify the effect of replacing that policy.

\item \textbf{Public-agent transfer.} CESS combines predictions for every candidate with logged probabilities of document selection and of reaching each round, shrinks high-variance corrections from short searches, and reports an interval when some documents cannot be sampled. Across 216 trajectories forming 108 paired comparisons in a public Open Deep Research agent, it reduces mean absolute error against the candidate-pool target by $60.1\%$ and ranking sensitivity by $87.2\%$ relative to the Opened Mean.

\item \textbf{Correction is not attribution.} Across two evidence-synthesis benchmarks and 4,800 LLM-agent trajectories under paired interventions, we show that target correction and search-policy effects answer different questions. CESS estimates the pool target, while interventions measure how selection and stopping change opened evidence.
\end{enumerate}

\section{What Does Each Search Outcome Measure?}
\label{sec:causal_modeling}

\subsection{Candidate-Pool Target and Search Process}

For a research question $x$, let $\mathcal D(x)=\{d_1,\ldots,d_N\}$ be a prespecified candidate pool and $Y_i\in[-1,1]$ the evidence score of document $d_i$. Positive and negative scores indicate support for opposite directions of the focal claim or outcome, while the magnitude reflects the dataset-provided evidence strength. The pool target is the equally weighted evidential conclusion
\begin{equation}
\theta^{*}=\frac{1}{N}\sum_{i=1}^{N}Y_i.
\label{eq:target_estimand}
\end{equation}
Appendix~\ref{app:evidence_scores} gives the dataset-specific score construction for MS2 and PERSPECTRUM. An agent observes only a trajectory of selected documents, whose distribution depends on its history. We summarize this feedback with the sequential structural model~\citep{pearl1995causal}
\begin{equation}
\begin{aligned}
Q_t &= f_Q(H_t,U_t^Q), \quad E_t = f_E(Q_t,G_t,\mathcal{D},U_t^E), \quad A_t = f_A(H_t,Q_t,E_t,U_t^A),\\
H_{t+1} &= f_H(H_t,Q_t,E_t,A_t,U_t^H), \quad C_t = f_C(H_{t+1},U_t^C).
\end{aligned}
\label{eq:sequential_scm}
\end{equation}
Here $H_t$ is the search history, $Q_t$ the query, $G_t$ the ranking condition, $E_t$ the exposed set, $A_t$ the selected document, and $C_t$ the stopping decision. Thus early evidence changes both the current opened-evidence conclusion and future queries, selection probabilities, and stopping.

The target in Eq.~\ref{eq:target_estimand} asks whether a search trajectory represents a candidate pool fixed before evaluation. It is an audit target for that pool, not a claim about the entire open web or the ultimate truth of the research question. Fixing the documents and aggregation rule makes acquisition bias auditable across ranking and stopping policies. Equal weighting is our benchmark convention. The framework also admits prespecified quality or source-deduplication weights.

\subsection{How Search Policies Change Conclusions}
\label{sec:explanation_target}

Let $O(g)=\mathbb E[\widehat\theta_{\mathrm{opened}}(g)\mid\mathcal D]$ denote the expected Opened Mean under search policy $g$, and let $\tau_O(g,g')=O(g)-O(g')$. Unlike $\theta^*$, $\tau_O$ measures how replacing the search policy changes the evidential conclusion formed from opened documents. We measure this effect with a $2\times2$ intervention crossing uniform versus agent document selection ($\pi_0,\pi_1$) with fixed versus adaptive stopping ($\sigma_0,\sigma_1$):
\begin{equation}
\theta_{ab}=\mathbb E[\widehat\theta_{\mathrm{opened}}\mid
do(\pi=\pi_a),do(\sigma=\sigma_b)],\qquad a,b\in\{0,1\},
\label{eq:factorial_targets}
\end{equation}
\begin{equation}
\Delta_{\mathrm{sel}}=\theta_{10}-\theta_{00},\quad
\Delta_{\mathrm{stop}}=\theta_{01}-\theta_{00},\quad
\Delta_{\mathrm{int}}=\theta_{11}-\theta_{10}-\theta_{01}+\theta_{00}.
\label{eq:mechanism_decomposition}
\end{equation}
These effects sum to $\theta_{11}-\theta_{00}$ by construction. They are defined by the interventions, not by CESS. We estimate them from paired runs.

\paragraph{Three questions, three quantities.} $\theta^*$ asks what conclusion the prespecified evidence pool supports under the chosen aggregation rule. $O(g)$ asks what conclusion a typical execution of policy $g$ forms from the evidence it opens. $\tau_O(g,g')$ asks how that opened-evidence conclusion changes when $g'$ is replaced by $g$. CESS targets the first quantity. Repeated executions describe the second. Paired interventions or full-trajectory off-policy evaluation are needed for the third. Keeping these questions separate prevents an accurate audit estimate from being misread as a causal explanation of the agent.

For example, suppose supporting-first and opposing-first rankings search the same candidate pool. A large gap between the average scores of documents read under each ranking is a real ranking-policy effect. If CESS corrects both search runs toward the same $\theta^*$, the corrected gap should approach zero. That invariance is evidence of successful target correction, but it cannot show that ranking had little effect on what the agent actually observed. Indeed, using the corrected gap as the policy-effect estimate would erase the effect by construction.

\section{Causal Evidence Selection Correction}
\label{sec:method}

Causal Evidence Selection Correction (CESS) estimates $\theta^*$ from a known candidate pool by combining prediction with probability correction. For each document, $X_i$ contains features available before the evaluated trajectory, and the outcome model $m_\phi(X_i)$ predicts its evidence score $Y_i$. These predictions cover documents the agent never opens. For an opened document, the observed prediction error is then reweighted by its logged probability of selection and of reaching that round. The prediction is fixed before search, and the estimator observes $Y_i$ only if the agent opens $d_i$. The term \emph{causal} refers to the explicitly modeled acquisition process and intervention-defined policy effects. CESS itself targets the policy-invariant pool conclusion. This is a finite-pool estimation problem related to active testing~\citep{kossen2021activetesting}.

\subsection{Sequential Correction}

Let $K$ be the maximum budget, $T\leq K$ the observed rounds, $J_t$ the document selected at round $t$, and $S_t=1$ the event that round $t$ is reached. With history $H_t$, query $Q_t$, ranking $G_t$, and pre-selection stopping probability $c_t$, define
\begin{equation}
p_t=P\!\left(J_t=j_t\mid H_t,Q_t,G_t,S_t=1\right),\qquad
s_t=\prod_{k=1}^{t}(1-c_k),\qquad c_1=0.
\label{eq:selection_and_continuation}
\end{equation}
Here $j_t$ is the realized value of $J_t$, $s_t$ is the probability of reaching round $t$ along the observed history, and $s_t=1$ for fixed-length search. The logged values used by the estimator are denoted $\widehat p_t$ and $\widehat s_t$. Let $\overline m=N^{-1}\sum_i m_\phi(X_i)$ be the full-pool outcome-model prediction. The raw estimator is
\begin{equation}
\widehat{\theta}_{\mathrm{DR}}^{\mathrm{raw}}
=\overline{m}+\frac{1}{K}\sum_{t=1}^{T}
\frac{Y_{J_t}-m_{\phi}(X_{J_t})}{N\widehat{p}_t\widehat{s}_t}.
\label{eq:sequential_dr}
\end{equation}
In Eq.~\ref{eq:sequential_dr}, the numerator is the observed prediction error for the selected document. The factors $1/\widehat p_t$ and $1/\widehat s_t$ correct, respectively, selective document choice and the loss of later observations through stopping. The full-pool prediction $\overline m$ covers unopened documents and reduces residual variance. This construction follows augmented inverse-probability estimation~\citep{cassel1976difference,robins1994regression,dudik2011doubly}.

\begin{assumption}[Logged sequential design and overlap]
The outcome prediction for every candidate is fixed before the evaluated trajectory. The logged selection and continuation probabilities equal those used by the controller. At every reachable history, each candidate document has positive selection probability, and every evaluated round has positive probability of being reached.
\end{assumption}

\begin{theorem}[Design unbiasedness]
\label{thm:design_unbiasedness}
Under the logged-design and overlap assumption, conditional on the finite candidate pool,
\begin{equation}
\mathbb{E}\!\left[\widehat{\theta}_{\mathrm{DR}}^{\mathrm{raw}}\mid\mathcal D\right]=\theta^*.
\end{equation}
\end{theorem}
The proof is in Appendix~\ref{app:proofs}. Correct design probabilities and overlap are sufficient even when the outcome model is misspecified. We evaluate estimated probabilities and the practical clipping operation $\Pi_{[-1,1]}(\widehat\theta_{\mathrm{DR}}^{\mathrm{raw}})$ empirically~\citep{hadad2021adaptive,cook2024adaptive}.

The correction has a round-wise interpretation. Conditional on reaching round $t$, inverse selection weighting makes the observed residual represent the mean residual over the candidate pool. Inverse continuation weighting restores the contribution of trajectories censored before that round, and $1/K$ averages over the $K$ potential rounds. Adding $\overline m$ therefore recovers the pool target in expectation. The guarantee rests on two auditable conditions. Controller probabilities are logged, and every candidate document and evaluated round has positive probability.

\subsection{Why Correction Is Not Policy Attribution}
\label{sec:diagnostic_guarantees}

\begin{theorem}[Invariance--attribution incompatibility]
\label{thm:invariance_attribution}
Call a policy \emph{supported} if it satisfies the overlap conditions in Assumption~1. Let $\widetilde\theta(g)$ be any estimator satisfying $\mathbb E[\widetilde\theta(g)\mid\mathcal D]=\theta^*$ for every supported policy $g$. Then, for any two supported policies $g$ and $g'$,
\begin{equation}
\mathbb E[\widetilde\theta(g)-\widetilde\theta(g')\mid\mathcal D]=0.
\end{equation}
\end{theorem}
Consequently, a contrast of policy-invariant corrected targets identifies $\tau_O(g,g')$ only when $\tau_O(g,g')=0$.

Invariance is desirable for estimating $\theta^*$, but it is incompatible with recovering a nonzero opened-evidence effect. Moreover, when selection changes future queries, candidate-selection and continuation probabilities omit part of the trajectory likelihood ratio. Policy attribution therefore requires direct interventions or a full off-policy estimator covering query generation, retrieval, selection, and stopping.

\subsection{Finite-Budget Shrinkage}

The unshrunk sequential correction can vary greatly when only a few documents are observed. CESS stabilizes it by shrinking toward the full-pool prediction, following the bias--variance logic of doubly robust off-policy evaluation~\citep{su2020shrinkage}. The coefficient \(\lambda\in[0,1]\) controls how much of the probability correction is retained, and $\Pi_{[-1,1]}$ clips its argument to the valid score range:
\begin{equation}
\widehat{\theta}_{\mathrm{CESS}}
=\Pi_{[-1,1]}\!\left[\overline{m}+\lambda
\left(\widehat{\theta}_{\mathrm{DR}}^{\mathrm{raw}}-\overline{m}\right)\right].
\label{eq:anchored_cess}
\end{equation}
At $\lambda=0$, CESS uses the full-pool prediction. At $\lambda=1$, it uses the unshrunk sequential correction. Intermediate values trade prediction error against correction variance. Equation~\ref{eq:anchored_cess} defines the primary MS2 estimator. Every CESS variant retains the same logged sequential correction and differs only in its shrinkage anchor, coefficient, or clipping order. Appendix~\ref{app:cess_configurations} maps each experiment to its exact formula.

MS2 uses coefficients $0.5484$ for Qwen2.5-32B and $0.5649$ for OLMo3-7B. \emph{Task-wise CESS} is a variant that selects a separate coefficient for each question using only information available before document scores are observed. In the public-agent and ten-round PERSPECTRUM analyses, we first clip the raw correction to the valid score range and then shrink it toward the full-pool prediction with $\lambda=0.5$. The five-round PERSPECTRUM analysis instead shrinks from the Opened Mean. All settings are fixed before evaluation.

\subsection{Limited Support: An Interval Instead of a Point Estimate}
\label{sec:support_failure}

Let $\mathcal D_+$ contain the candidates with positive acquisition probability under the logged design, and let $\rho=1-|\mathcal D_+|/N$ be the inaccessible target mass for the equally weighted target. If $Y_i\in[L,U]$ and $\theta_+$ is the mean over reachable candidates, then
\begin{equation}
\theta^*\in\left[(1-\rho)\theta_++\rho L,\;(1-\rho)\theta_++\rho U\right].
\label{eq:support_bounds}
\end{equation}
These bounds are sharp without additional assumptions about inaccessible outcomes. Their width is $\rho(U-L)$. For our $[-1,1]$ score it is $2\rho$. When the reachable set changes, we propagate the bounds to the target estimate and report the resulting interval. If the inaccessible mass is unknown, $\rho$ can instead be varied in a sensitivity analysis.

\section{Experiments}
\label{sec:experiments}

\paragraph{Experimental plan.} Our experiments answer three questions. First, does the logged probability correction recover the candidate-pool target (RQ1)? Second, after finite-budget stabilization, does CESS improve accuracy while reducing dependence on document order (RQ2)? Third, do CESS contrasts and paired interventions behave differently, as predicted by Theorem~\ref{thm:invariance_attribution} (RQ3)? Appendix~\ref{app:experimental_details} gives score construction, estimator configurations, data separation, and implementation details.

The main target-recovery experiment uses 200 MS2~\citep{deyoung2021ms2} test questions containing 3,169 candidate studies. Qwen2.5-32B-Instruct~\citep{qwen2024qwen25} and OLMo3-7B-Instruct~\citep{olmo2025olmo3} each search for four rounds under three document orders: the original ranking, supporting evidence first, and opposing evidence first. This design produces 1,200 trajectories. We additionally evaluate 216 four-search trajectories from a public agent on 36 PERSPECTRUM tasks. The appendix reports a complementary five-round PERSPECTRUM evaluation with the same 36 tasks and two search controllers (Appendix~\ref{sec:generalization}). A separate $2\times2$ selection--stopping experiment contributes 4,800 trajectories, organized as 1,200 paired blocks containing all four intervention conditions.

All method settings are fixed without using evidence scores from the test questions. Confidence intervals resample whole questions and keep every associated trajectory together. This prevents multiple runs of the same question from being treated as independent observations.

The uncorrected \method{Opened Mean} averages evidence scores only over opened documents. Mean absolute error (MAE) measures deviation from the pool target in Eq.~\ref{eq:target_estimand}. Ranking sensitivity is the absolute change in an estimate between supporting-first and opposing-first rankings of the same candidate pool. Paired intervention effects answer a different question: they measure the signed change in Opened Mean when a selection or stopping policy is replaced. We report accuracy and ranking sensitivity together because a constant estimator can be perfectly stable yet inaccurate.

\paragraph{Identification and support diagnostics.} Controlled tests confirm that the logged selection and continuation probabilities remove the corresponding bias in the unshrunk estimator. Appendix~\ref{sec:controlled_mechanisms} further evaluates limited support and perturbations to predictions or probabilities.

\subsection{Target Recovery and the Accuracy--Stability Trade-off}
\paragraph{Ranking intervention.}\label{sec:ranking_intervention} We keep the candidate documents fixed and change only their order. Supporting evidence appears first in one condition and opposing evidence appears first in the other. Ranking sensitivity is the absolute difference between the two resulting estimates. The simulation uses 30 documents and at most four search rounds. Appendix~\ref{app:ranking_simulation} gives the full design.

In the simulation, the unshrunk correction underlying CESS reduces the difference between the two rankings from 0.7478 to 0.0014, a 99.8\% reduction. On MS2 review questions, Task-wise CESS chooses its shrinkage coefficient separately for each question and reduces that difference from 0.0415 to 0.0185 with Qwen2.5-32B and from 0.0391 to 0.0076 with OLMo3-7B, reductions of 55.4\% and 80.6\%, respectively.

Figure~\ref{fig:search_bias} shows both results. The simulated panel isolates the unshrunk probability correction, while the MS2 panel evaluates stabilized Task-wise CESS on real review questions. In both panels, the candidate evidence is identical across rankings. Shrinkage alone can also make an estimator less sensitive to ranking, so the next comparison holds accuracy constant.

\begin{figure}[t]
    \centering
    \includegraphics[width=0.75\linewidth]{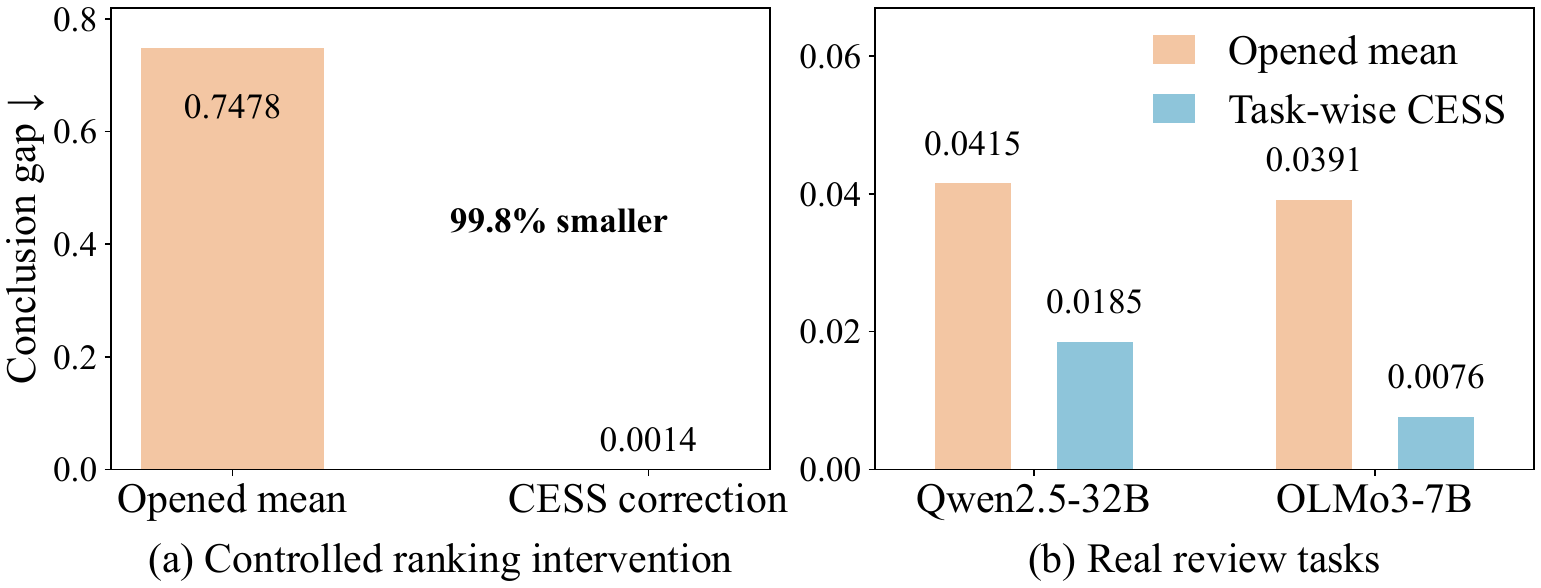}
\caption{Ranking sensitivity with the candidate evidence held fixed. \textbf{(a)} In a controlled simulation, the unshrunk sequential correction reduces sensitivity from 0.7478 to 0.0014 before shrinkage and clipping. \textbf{(b)} On MS2 review questions, \method{Task-wise CESS} reduces sensitivity for both query models. Lower is better.}
    \label{fig:search_bias}
\end{figure}

\paragraph{Pool-target recovery.}\label{sec:estimator_comparison} We compare prediction-only estimates, the uncorrected average over read documents, sequential probability corrections, and CESS. The ``Tuning-set Mean'' assigns every test question the same average target from the tuning questions, whereas ``Outcome Regression'' (OR) predicts every candidate document before search and averages the predictions over the full pool. ``Opened Mean'' uses only the documents read by the agent. Sequential IPW corrects each opened document using its selection and continuation probabilities, while Sequential Doubly Robust adds a probability-weighted residual correction to the OR prediction. CESS stabilizes this sequential correction through shrinkage fixed before test outcomes are observed. All methods use the same tuning questions and fixed test-time settings. We average results within each question and then weight questions equally. MAE measures target error. Ranking sensitivity measures how much an estimate changes when the same evidence pool is reordered. Table~\ref{tab:estimator_comparison} reports the resulting accuracy and ranking-sensitivity comparison.

\begin{table}[t]
    \centering
    \small
    \setlength{\tabcolsep}{7pt}
\caption{Accuracy and sensitivity to document order on MS2, averaged across the two query models. Shrinkage settings are chosen on separate tuning questions and kept unchanged on the test questions. Worst-ranking MAE is error under the less favorable of the two rankings. Lower is better.}
    \label{tab:estimator_comparison}
    \begin{tabular}{lccc}
        \toprule
        Method
        & MAE
        & Ranking sensitivity
        & Worst-ranking MAE \\
        \midrule
        Tuning-set Mean & 0.2733 & 0.0000 & 0.2733 \\
        Outcome Regression & 0.2571 & 0.0000 & 0.2571 \\
        Opened Mean & 0.1743 & 0.2445 & 0.2522 \\
        Sequential IPW & 0.1883 & 0.2573 & 0.2727 \\
        Sequential Doubly Robust & 0.1893 & 0.2626 & 0.2763 \\
        \methodname{CESS} & \textbf{0.1583} & \textbf{0.1481} & \textbf{0.2165} \\
        \bottomrule
    \end{tabular}
\end{table}

CESS achieves the best overall results among estimators that use observed opened-document outcomes, attaining the lowest MAE, ranking sensitivity, and worst-ranking MAE. Relative to Opened Mean, CESS reduces MAE by $9.2\%$, ranking sensitivity by $39.4\%$, and worst-ranking MAE by $14.2\%$. Compared with Sequential Doubly Robust, it reduces these three metrics by $16.4\%$, $43.6\%$, and $21.6\%$, respectively. Prediction-only baselines have zero ranking sensitivity because they ignore the realized search trajectory, whereas CESS combines the observed evidence with sequential probability correction to achieve both accurate pool-target recovery and stable conclusions under different document orders.

\paragraph{Comparing methods at equal accuracy.}\label{sec:pareto_frontier} Stronger shrinkage can make any method appear more stable by moving it closer to a fixed prediction. We therefore sweep the same coefficient range for CESS and a simple shrinkage baseline. We then compare ranking sensitivity at the same MAE, using interpolation only when neighboring settings bracket the target value. Question-level bootstrap resampling repeats this matching procedure to quantify uncertainty.

\begin{table}[t]
\centering
\small
\setlength{\tabcolsep}{3.2pt}
\caption{Ranking sensitivity at matched MAE across 28 dataset--model--budget--prediction settings. Each count records the direction of the paired comparison using a 95\% question-level bootstrap interval. No multiplicity adjustment is applied.}
\label{tab:matched_frontier}
\begin{tabular}{lccc}

\toprule
Comparison & \shortstack{CESS lower\\sensitivity} & \shortstack{No detected\\difference} & \shortstack{CESS higher\\sensitivity} \\
\midrule
At matched MAE & 10 & 18 & 0 \\
\bottomrule
\end{tabular}
\end{table}

Table~\ref{tab:matched_frontier} summarizes the matched-MAE comparison. The 95\% interval favors lower ranking sensitivity for CESS in 10 of 28 comparisons. The other 18 intervals include zero, and none favors higher sensitivity for CESS. The gains concentrate on PERSPECTRUM, where ranking more strongly changes the evidence opened. This comparison isolates the benefit of probability correction from the generic stabilizing effect of shrinkage. Appendix~\ref{app:pareto_robustness} reports the reciprocal matched-sensitivity comparison and perturbation tests.

\paragraph{Limited budgets.}\label{sec:budget_efficiency} This version of CESS chooses its shrinkage coefficient separately for each question using information available before its document values are observed. After one search round, it reduces MAE from $0.3687$ to $0.2157$ for Qwen and from $0.3522$ to $0.2090$ for OLMo. At four rounds, the gains are $9.1\%$ and $15.8\%$. Under adaptive stopping, the gains are $11.9\%$ and $15.6\%$. These improvements compare CESS with using only the documents the agent read. Full curves appear in Figure~\ref{fig:budget_efficiency}.

\begin{figure}[t]
    \centering
    \includegraphics[width=0.7\linewidth]{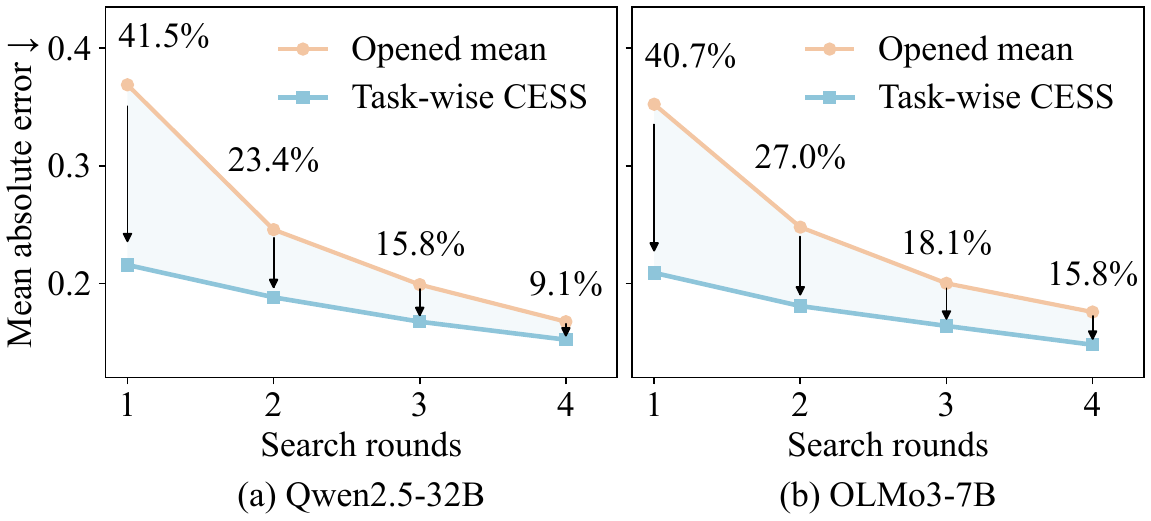}
\caption{MAE across search budgets for \method{Opened Mean} and \method{Task-wise CESS}. Percentages report relative reductions within a round. Gains are largest when only one or two documents can be opened.}
    \label{fig:budget_efficiency}
\end{figure}

\paragraph{Public-agent transfer.}\label{sec:public_agent_transfer} We integrate CESS with the public Open Deep Research agent and leave its planner, adaptive query generator, evidence-use process, and final synthesis unchanged. We replace only its search interface, allowing the agent to search the PERSPECTRUM documents while we record all available documents and the probability of selecting each one. The evaluation contains 36 tasks, two rankings, and three paired seeds, for 216 trajectories of four searches. Within each pair, both branches start from the same pre-intervention state and use the same Gumbel draws. All logged probability vectors are positive and normalized. They exactly reproduce the selected documents. The primary short-budget analysis uses the first $K=2$ selections. Appendix~\ref{app:public_agent} also reports $K=4$. Direction error records whether the estimate and pool target fall into different positive, neutral, or negative categories.

\begin{table}[t]
\centering
\small
\setlength{\tabcolsep}{7pt}
\caption{Transfer to a public Open Deep Research agent on 36 PERSPECTRUM tasks, two evidence rankings, and three seeds (216 trajectories). Results use the primary budget $K=2$. Lower is better.}
\label{tab:public_agent_transfer}
\begin{tabular}{lccc}
\toprule
Method & MAE & Ranking sensitivity & Direction error \\
\midrule
Opened Mean & 0.5727 & 0.5849 & 0.4907 \\
Outcome Regression & 0.2404 & \textbf{0.0000} & 0.3333 \\
Sequential Doubly Robust & 0.3053 & 0.1492 & 0.3287 \\
\methodname{CESS} & \textbf{0.2287} & 0.0746 & \textbf{0.3102} \\
\bottomrule
\end{tabular}
\end{table}

CESS attains the lowest MAE and direction-error rate in Table~\ref{tab:public_agent_transfer}. Among methods that update with opened-document outcomes, it also has the lowest ranking sensitivity. Relative to Opened Mean, CESS reduces MAE by $60.1\%$, ranking sensitivity by $87.2\%$, and direction errors by $18.1$ percentage points. The 95\% CIs for the absolute reductions are $[0.2614,0.4238]$ for MAE and $[0.3560,0.6723]$ for ranking sensitivity. CESS also improves over the unshrunk sequential doubly robust estimator on both metrics, with paired intervals excluding zero. Appendix~\ref{app:public_agent} gives the category threshold and comparisons across search budgets.

The paired branches share their first query, but their later trajectories can diverge after they read different evidence. By round four, $63.9\%$ generate different queries and $75.0\%$ select different document sequences. CESS therefore transfers to trajectories with adaptive query and evidence feedback, rather than only to a fixed sequence of document choices.

\subsection{Why Policy Attribution Requires Intervention}
\label{sec:true_agent_audit}

We test Theorem~\ref{thm:invariance_attribution} on 200 MS2 questions per query model using a $2\times2$ intervention that crosses agent versus uniform document selection with agent-controlled versus fixed ten-round stopping. Each condition is repeated three times with paired seeds, yielding 4,800 LLM-agent trajectories with logged selection and continuation probabilities. Agent-controlled runs average $5.0$--$5.1$ rounds. Policy effects are computed from the four resulting conclusions using Eq.~\ref{eq:mechanism_decomposition}. Table~\ref{tab:true_agent_faithfulness} summarizes the repeatability of these effects and the relationship between intervention effects and CESS contrasts.

\begin{table}[t]
\centering
\small
\setlength{\tabcolsep}{5.5pt}
\caption{Paired LLM search experiment on 200 questions per model. Intraclass correlation (ICC) measures whether the directly observed effects repeat across three runs. Contrast MAE and rank correlation compare the difference between two CESS estimates with the directly observed selection effect. ``Adaptive'' lets later queries respond to selected evidence. ``Fixed'' replays a common query sequence. Higher ICC and rank correlation, and lower MAE, are better.}
\label{tab:true_agent_faithfulness}
\begin{tabular}{lccccc}
\toprule
Model & Sel. ICC & Stop. ICC & \shortstack{Contrast MAE\\(adaptive)} & \shortstack{Rank corr.\\(adaptive)} & \shortstack{Contrast MAE\\(fixed)} \\
\midrule
OLMo3-7B & 0.679 & 0.099 & 0.151 & 0.269 & 0.193 \\
Qwen2.5-32B & 0.685 & 0.102 & 0.146 & 0.236 & 0.197 \\
\bottomrule
\end{tabular}
\end{table}

The directly measured selection effect is moderately repeatable, with ICCs of $0.679$ for OLMo and $0.685$ for Qwen. However, the differences between two CESS estimates have rank correlations of only $0.269$ and $0.236$ with these intervention effects. This is the separation predicted by Theorem~\ref{thm:invariance_attribution}. CESS estimates what the same task-specific candidate pool supports. Paired interventions measure how changing the policy alters the evidence actually opened. They are complementary audit outputs, not interchangeable mechanism scores. Appendix~\ref{app:new_experiments} gives the full design and repeatability analysis.

\section{Related Work}
\label{sec:related_work}

\paragraph{Deep Research agents and report evaluation.} STORM and Co-STORM organize research through multiple perspectives~\citep{shao-etal-2024-assisting,jiang-etal-2024-unknown}. DeepResearcher learns multi-step web search~\citep{zheng-etal-2025-deepresearcher}. HypoSearch explores alternative directions before commitment~\citep{zhou2026hyposearch}. Complementary benchmarks evaluate report quality, citations, factuality, coverage, and logical support~\citep{du2026deepresearchbench,NEURIPS2025_fdcec9f5,huang-etal-2026-deepfact,avraham-etal-2026-dream,zhao-etal-2026-reportlogic,chen-etal-2026-beyond-single}. These works improve evidence acquisition or assess the completed report. They do not estimate a common candidate-pool conclusion from one adaptively selected trajectory. Accurate citations can coexist with an unrepresentative evidence sample.

\paragraph{Adaptive estimation and our gap.} Inverse-probability and doubly robust estimators correct selective observation when acquisition probabilities are available~\citep{cassel1976difference,robins1994regression,dudik2011doubly}. Active testing similarly estimates a target under adaptive label acquisition~\citep{kossen2021activetesting}. Deep Research is harder because earlier evidence changes later queries, document selection, and stopping. It also raises two different audit questions: what does the candidate evidence support, and what did the search policy change? CESS adapts probability correction to the logged search process, stabilizes it for short budgets, reports bounds when some documents have zero probability, and uses interventions for policy attribution.

\section{Conclusion}

We formulate Deep Research search as adaptive evidence sampling and separate two audit goals. The first is to estimate the conclusion supported by a specified candidate pool. The second is to measure how a search policy changes the evidence opened. CESS addresses the first goal by combining predictions for the full pool with logged probabilities of document selection and of reaching each round. It shrinks high-variance corrections from short searches and reports an interval when some documents cannot be sampled. Its gains in estimating the candidate-pool target are largest under short budgets. In a public Open Deep Research agent, CESS reduces target error by $60.1\%$ and ranking sensitivity by $87.2\%$ relative to the Opened Mean. At matched target error, 10 of 28 intervals favor lower ranking sensitivity for CESS, and none favors higher sensitivity. Our theorem and 4,800 intervention trajectories address the second goal. CESS estimates what the candidate evidence supports, whereas paired interventions measure what the search policy changed.

\subsection*{AI use statement}

In this work, we used generative AI tools only for language polishing, including improving grammar, clarity, and fluency of the manuscript. We have not used generative AI tools for generating research ideas, developing the methodology, designing experiments, analyzing results, or drawing scientific conclusions. Other disclosure categories are not applicable to this work. We reviewed all AI-assisted edits and verified that they did not alter the technical content or scientific claims. We take full responsibility for the final content of this work, including all text, claims, and artifacts.

\subsection*{Ethics Statement}
All experiments are conducted in controlled benchmark environments and do not involve human subjects or private user data. We do not identify additional ethical risks beyond those commonly associated with LLM agent research.

\subsection*{Reproducibility Statement}
The paper specifies the main causal variables, interventions, measurements, models, benchmarks, coefficient choices, and experimental procedures. The public-agent experiment logs a complete candidate probability vector at every search round and stores the paired pre-intervention state and common random numbers needed for computational replay. 

\bibliography{iclr2027_conference}
\bibliographystyle{iclr2027_conference}

\clearpage
\appendix

\section{Evidence Construction and Experimental Details}
\label{app:experimental_details}

\subsection{Construction of Document Evidence Scores}
\label{app:evidence_scores}

\paragraph{MS2.} Each review task specifies a target intervention--outcome pair. For each candidate study, we select the matching significance record supplied with the data. If more than one record matches, we use the one with the highest evidence-sentence score, breaking ties by the normalized evidence text. Let \(p_i^{+}\), \(p_i^{0}\), and \(p_i^{-}\) denote the supplied probabilities of a significantly increased outcome, no significant difference, and a significantly decreased outcome, respectively. The evidence value is
\(Y_i=p_i^{+}-p_i^{-}\). The implementation checks that the probabilities lie in \([0,1]\) and sum to within \(0.02\) of one, then normalizes their sum when necessary. A positive score indicates an increase in the target outcome, and a negative score indicates a decrease. These scores encode reported outcome direction rather than effect magnitude or clinical benefit. A zero score can arise from equal probabilities of increase and decrease as well as from a high probability of no significant difference.

\paragraph{PERSPECTRUM.} Each task corresponds to a claim. We assign \(+1\) to a SUPPORT label and \(-1\) to an UNDERMINE label. An evidence document may be linked to several labeled perspectives for the same claim. Let \(\mathcal{A}_i\) contain these perspective--evidence associations and let \(\ell_a\in\{-1,+1\}\) be the corresponding stance label. We compute
\begin{equation}
Y_i=\frac{1}{|\mathcal{A}_i|}\sum_{a\in\mathcal{A}_i}\ell_a.
\end{equation}
Thus, a document linked exclusively to supporting perspectives receives \(+1\), one linked exclusively to opposing perspectives receives \(-1\), and mixed associations yield intermediate values. Each candidate document contributes once to the candidate-pool target, irrespective of its number of associations.

\paragraph{Reference conclusion.} For both datasets, the reference is \(\theta^{*}=N^{-1}\sum_{i=1}^{N}Y_i\), with equal weight for each candidate document. Document scores are constructed from the supplied dataset fields before trajectory evaluation. The resulting reference measures the aggregate direction encoded by the candidate pool and provides a common target for all estimators within a task.

\subsection{Observation Units and CESS Configurations}
\label{app:cess_configurations}

A task defines one research question and its candidate pool. A trajectory records the sequence of document selections under a specified search configuration. Each selection round contributes one observation, and a document selected repeatedly contributes at each selection. Accordingly, a trajectory with \(T\) rounds contains \(T\) observations but may contain fewer than \(T\) distinct documents. The uncorrected baseline is \(\widehat{\theta}_{\mathrm{naive}}=T^{-1}\sum_{t=1}^{T}Y_{J_t}\).

The main MS2 comparison contains 200 tasks and three ranking conditions for each of two query-generation models, giving \(200\times3\times2=1{,}200\) trajectories. Each trajectory contains four selection rounds. These are repeated search evaluations of 200 tasks, rather than 1,200 distinct research questions. The MS2 ranking, budget, and stopping analyses use a separate set of 200 tasks.

Let \(\overline{m}\) denote the mean outcome prediction over the candidate pool and let \(\widehat{\theta}_{\mathrm{DR}}^{\mathrm{raw}}\) denote the sequential estimator before clipping to \([-1,1]\). All configurations use the probability correction in Eq.~\ref{eq:sequential_dr}. They differ only in the estimate toward which they shrink (the anchor), the shrinkage coefficient, or the order of shrinkage and clipping. The primary MS2 configuration is Eq.~\ref{eq:anchored_cess}, with one coefficient per query model: \(0.5484\) for Qwen2.5-32B and \(0.5649\) for OLMo3-7B. Task-wise CESS uses the same formula but predicts one coefficient per task before observing its document scores. It is used in the MS2 ranking, budget, and adaptive-stopping analyses.

The five-round PERSPECTRUM evaluation in Appendix~\ref{sec:generalization} starts from a different anchor:
\begin{equation}
\widehat{\theta}_{\mathrm{CESS}}=\Pi_{[-1,1]}\!\left[
\widehat{\theta}_{\mathrm{naive}}+0.20
(\widehat{\theta}_{\mathrm{DR}}-\widehat{\theta}_{\mathrm{naive}})\right],
\end{equation}
where \(\widehat{\theta}_{\mathrm{DR}}=\Pi_{[-1,1]}[\widehat{\theta}_{\mathrm{DR}}^{\mathrm{raw}}]\). This version moves the Opened Mean toward the clipped correction. Its coefficient was chosen before evaluating the 36 topics.

We choose each query model's shared coefficient using separate tuning runs and then keep it fixed during evaluation. The task-wise coefficient uses only information available before the agent reads the documents to predict how variable the correction will be. Neither procedure uses evidence scores from the evaluated question. The released implementation reports the data split, tuning criterion, and coefficient calculation.

The ten-round PERSPECTRUM experiment and the public-agent experiment start from the full-pool prediction. Both retain the preselected coefficient \(\lambda=0.5\). They first restrict the probability correction to the valid score range and then shrink it toward the prediction:
\begin{equation}
\widehat{\theta}_{\mathrm{CESS}}^{(10)}=\Pi_{[-1,1]}\!\left[
\overline m+0.5\left(\Pi_{[-1,1]}[
\widehat{\theta}_{\mathrm{DR}}^{\mathrm{raw}}]-\overline m\right)\right].
\label{eq:horizon10_cess}
\end{equation}
For the ten-round experiment, Eq.~\ref{eq:sequential_dr} always uses maximum budget \(K=10\), even when a trajectory stops earlier. The public-agent experiment uses the same formula that clips before shrinking for prefixes \(K\in\{1,2,4\}\). There, $s_t=1$ because every trajectory completes four searches. The primary MS2 configuration uses the opposite order: it shrinks first and clips afterward. Recomputing Eq.~\ref{eq:horizon10_cess} from the saved predictions and clipped sequential estimates reproduces all 2,592 ten-round estimates exactly.

The ten-round implementation uses one common predicted value within each question, \(m_\phi(X_i)=\overline{m}\). Appendix~\ref{app:horizon10_details} recalculates both orders of clipping and shrinkage from the saved runs. It separates the reported estimator, whose settings were kept unchanged, from the version that starts with the unshrunk probability correction.

\subsection{Identification and Support Diagnostics}
\label{sec:controlled_mechanisms}

We first test the probability correction without shrinkage. Each condition contains 500 simulated candidate pools and 12 search runs per pool. Document features and evidence scores are generated once and then held fixed across selection and stopping conditions. Uncertainty is therefore calculated across pools. Repeated runs of the same pool are not treated as independent. The search policy never sees the evidence scores of unopened documents. In Table~\ref{tab:mechanism_bias}, ``Biased'' selection favors documents according to their evidence-related attributes, and ``Informative'' stopping depends on the observed search history.

\begin{table}[t]
\centering
\small
\caption{Bias in the unshrunk estimator under controlled selection and stopping. The comparison removes only the probability correction for the decision being tested. Parentheses give standard errors calculated across questions, with repeated runs of each question kept together.}
\label{tab:mechanism_bias}
\begin{tabular}{llcc}
\toprule
Selection & Stopping & Full correction & Mechanism removed \\
\midrule
No & No & $-0.0032\;(0.0031)$ & $-0.0032\;(0.0031)$ \\
No & Informative & $-0.0032\;(0.0052)$ & $-0.0086\;(0.0028)$ \\
Biased & No & $-0.0049\;(0.0034)$ & $\phantom{-}0.0802\;(0.0036)$ \\
Biased & Informative & $\phantom{-}0.0103\;(0.0091)$ & $\phantom{-}0.0876\;(0.0096)$ \\
\bottomrule
\end{tabular}
\end{table}

Table~\ref{tab:mechanism_bias} shows that each weight corrects the decision it was designed for. Removing selection weights introduces bias when the search favors documents related to their evidence values. Removing continuation weights matters when stopping depends on the search history. These tests validate the unshrunk probability correction. Shrinkage is applied later to stabilize short searches. Once tuning on separate questions selects $\lambda=0.01$, removing either weight changes MAE by only about $10^{-4}$ (Section~\ref{app:identification_details}).

\paragraph{Limited support.} We next make document choice deterministic, leaving only $3.6\%$ of candidate documents reachable on average. A point estimate can no longer recover the full-pool target from these trajectories. Instead, the interval in Eq.~\ref{eq:support_bounds} covers all 16,000 simulated targets. When exploration remains positive but small, CESS still combines observed prediction errors with candidate-level predictions. As overlap decreases, tuning places more weight on the prediction. Section~\ref{app:pareto_robustness} reports additional results.

\paragraph{Perturbations.} We vary both prediction quality and the accuracy of the selection probabilities. Probabilities estimated from visible ranking features give MAE $0.0570$, close to the $0.0576$ obtained with exact probabilities and the same linear outcome predictor. Under extremely limited overlap, tuning selects $\lambda=0$ and falls back to the full-pool prediction rather than amplifying unstable inverse-probability weights. Section~\ref{app:pareto_robustness} reports additional settings.

\subsection{Controlled Ranking Experiment}
\label{app:ranking_simulation}

We construct finite candidate pools in which each document has an evidence value, observable pre-selection features, and a ranking attribute correlated with that value. The two ranking conditions reverse the preference induced by this attribute while keeping the candidate documents and their evidence values fixed. Search remains adaptive: previously observed evidence affects subsequent queries and the probability of continuing, and a uniform exploration component gives every candidate document positive selection probability.

The experiment uses 30 candidate documents per task, a maximum budget of four selection rounds, and 20 random seeds with 400 test tasks per seed. Paired ranking conditions share the same task pools and random draws for document selection, query updates, and stopping. This pairing isolates changes caused by the ranking intervention rather than changes in candidate evidence or random sampling.

We compare \method{Opened Mean} with the unshrunk and unclipped sequential correction underlying CESS. The former produces a conclusion gap of 0.7478 between the two ranking conditions, whereas the latter reduces the gap to 0.0014, a 99.8\% reduction. This experiment tests whether the sequential correction removes ranking-induced shifts when selection and continuation probabilities are known.

\subsection{Additional Identification Diagnostics}
\label{app:identification_details}

\paragraph{Probability corrections after shrinkage.} On separate tuning questions, the chosen coefficient is $\lambda=0.01$. When the initial prediction, coefficient, search runs, and score-clipping rule are unchanged, removing selection or stopping probabilities changes MAE by only about $10^{-4}$. Table~\ref{tab:mechanism_bias} tests the unshrunk correction. With $\lambda=0.01$, little of that correction enters the final estimate. We therefore use the first experiment to test whether the weights remove bias and the accuracy and stability comparison to assess whether they help at the chosen coefficient.

\section{Proofs and Diagnostic Guarantees}
\label{app:proofs}

\subsection{Design unbiasedness and policy invariance}

\begin{proof}[Proof of Theorem~\ref{thm:design_unbiasedness}]
Let $R_t$ indicate that round $t$ is reached and write $e_i=Y_i-m_\phi(X_i)$. Conditional on a reachable history,
\begin{equation}
\mathbb E\!\left[\frac{e_{J_t}}{Np_t(J_t\mid H_t)}\,\middle|\,H_t,R_t=1\right]
=\frac{1}{N}\sum_{i=1}^{N}e_i.
\end{equation}
Inverse-continuation weighting gives $\mathbb E[R_t/s_t\mid\mathcal D]=1$. Therefore each of the $K$ potential rounds has expected weighted residual equal to $N^{-1}\sum_i e_i$, including rounds censored by stopping. Averaging over $K$ rounds and adding $\overline m=N^{-1}\sum_i m_\phi(X_i)$ yields $N^{-1}\sum_iY_i=\theta^*$.
\end{proof}

The statement applies to the raw estimator before clipping. Clipping, estimated probabilities, or an incorrect continuation model can introduce finite-sample dependence on the search policy.

\subsection{Invariance--attribution incompatibility}

\begin{proof}[Proof of Theorem~\ref{thm:invariance_attribution}]
Linearity of conditional expectation gives
\begin{equation}
\mathbb E[\widetilde\theta(g)-\widetilde\theta(g')\mid\mathcal D]
=\mathbb E[\widetilde\theta(g)\mid\mathcal D]
-\mathbb E[\widetilde\theta(g')\mid\mathcal D]=0.
\end{equation}
The opened-evidence effect is $\tau_O(g,g')=O(g)-O(g')$. Hence the corrected contrast equals that effect if and only if $\tau_O(g,g')=0$.
\end{proof}

The support interval in Eq.~\ref{eq:support_bounds} is sharp because the unreachable mass can attain either endpoint $L$ or $U$ without further restrictions.

\section{Paired Interventions on LLM Search}
\label{app:new_experiments}

\paragraph{Intervention design.} For each MS2 test question and model, we run all four combinations of agent versus uniform document choice and agent-controlled versus fixed-length search. We repeat each combination three times, pairing conditions within the same question and repetition. The four averages over documents read yield the selection, stopping, interaction, and total effects in Eq.~\ref{eq:mechanism_decomposition}. Every fixed-horizon run executes ten rounds. Each query model contributes $200\times4\times3=2{,}400$ trajectories, for 4,800 trajectories in total. These form 1,200 task--model--replicate blocks, each with four intervention conditions. Candidate probabilities are strictly positive. The smallest recorded probability is approximately $0.005$.

\paragraph{Repeatability of measured effects.} We estimate the three-run one-way intraclass correlation, ICC(1,3). Selection effects have ICC $0.679$ for OLMo and $0.685$ for Qwen. Total-effect ICC is $0.542$ and $0.639$. Stopping-effect ICC is $0.099$ and $0.102$. We therefore retain continuous intervention effects and their uncertainty rather than replacing them with a single per-task mechanism label.

\paragraph{Corrected contrasts and intervention effects.} The selection-policy effect compares agent and uniform selection under fixed stopping while allowing evidence to influence later queries. Its Spearman correlation with the contrast between CESS estimates is $0.269$ for OLMo and $0.236$ for Qwen. When later queries are held fixed by replaying draws on a common query sequence, the correlations are $-0.027$ and $0.086$. These results instantiate Theorem~\ref{thm:invariance_attribution}: a contrast designed to remove policy-dependent selection is a different estimand from the effect of changing that policy.

\section{Transfer to a Public Deep Research Agent}
\label{app:public_agent}

\paragraph{Protocol.} We integrate CESS with the public Open Deep Research agent and use Qwen2.5-32B-Instruct as its local language model. The fixed repository revision is \texttt{1b7d2e80db9faa586165c60e09096dbbfd483a64}. We retain the agent's research plan, adaptive query generation, evidence-use process, and final synthesis. We replace only the search tool with a fixed set of PERSPECTRUM documents. This makes every candidate document and selection probability observable. The supporting-first and opposing-first branches begin from the same checkpoint and first query and use the same Gumbel random vector at each round. Once they select different evidence, their later queries may diverge naturally. An exact-input cache fixes the language-model response for every repeated input, and the saved logs are sufficient to reconstruct every document selection and reported statistic.

\paragraph{Data separation and execution checks.} Ranking strength is selected on separate calibration tasks before the test configuration is fixed. Estimator performance is not used in this choice. The test contains 36 tasks, two rankings, and three seeds, giving 216 trajectories and 108 paired units. Every trajectory completes four searches and runs the public agent's synthesis component. Within each pair, the pre-intervention state and Gumbel draws are identical. Every logged selection-probability vector is positive, normalized, and consistent with the selected document. The saved state reconstructs every selection. The minimum candidate probability is $0.00979$, and the largest realized importance weight for the uniform target is $4.0187$.

\paragraph{Accuracy and ranking dependence.} Each trajectory contains four searches. The primary short-budget analysis uses its first two selections ($K=2$), while the secondary analysis uses all four ($K=4$). Direction error assigns both the estimate and the pool target to a positive, neutral, or negative category using thresholds $\pm0.05$ and then records whether the categories disagree. Table~\ref{tab:public_agent_transfer} reports $K=2$. Compared with Opened Mean, CESS lowers MAE from $0.5727$ to $0.2287$, ranking sensitivity from $0.5849$ to $0.0746$, and direction error from $0.4907$ to $0.3102$. The paired reductions are $0.3440$ in MAE (95\% CI: $[0.2614,0.4238]$), $0.5103$ in ranking sensitivity (95\% CI: $[0.3560,0.6723]$), and $18.1$ percentage points in direction error (95\% CI: $[4.17,31.48]$ points). Relative to the unshrunk sequential doubly robust estimator, CESS lowers MAE by $0.0766$ and ranking sensitivity by $0.0746$. Both paired intervals exclude zero. The median effective sample size (ESS) is $1.70$ at $K=2$ and $2.81$ at $K=4$, which motivates stabilization under short budgets.

\paragraph{Comparison with strong shrinkage.} At $K=2$, the matched-MAE sensitivity difference between CESS and outcome-regression (OR) shrinkage is $0.0100$ with a 95\% interval of $[-0.1780,0.0796]$. At $K=4$, CESS achieves $0.0461$ lower MAE than OR shrinkage at matched sensitivity (95\% CI: $[-0.0783,-0.0169]$). Together with Table~\ref{tab:public_agent_transfer}, these results establish transfer gains over opened evidence and unshrunk sequential estimation, and an accuracy gain over OR shrinkage at the longer budget.

\paragraph{Trajectory divergence and scope.} All paired trajectories begin with the same query. By round four, $63.9\%$ have generated different queries and $75.0\%$ have selected different document sequences. Position-wise document overlap is $68.5\%$ at $K=2$ and $66.4\%$ at $K=4$. The design thus evaluates evidence-selection correction after the ranking intervention has propagated through adaptive query generation and document choice. The fixed four-search horizon isolates this transfer test from the separate stopping analysis in Section~\ref{app:new_experiments}.

\section{Five-Round PERSPECTRUM Evaluation}
\label{sec:generalization}

This evaluation uses 36 non-medical PERSPECTRUM topics and up to five search rounds. Its CESS variant starts from the Opened Mean and moves $20\%$ toward the sequential correction: \(\widehat{\theta}_{\mathrm{CESS}}=\Pi_{[-1,1]}[\widehat{\theta}_{\mathrm{naive}}+0.20(\widehat{\theta}_{\mathrm{DR}}-\widehat{\theta}_{\mathrm{naive}})]\). The coefficient is selected on separate tuning questions and then fixed for all 36 evaluation topics. This differs from the ten-round variant, which starts from the full-pool prediction. We report MAE, root mean squared error (RMSE), average within-question bias, ranking sensitivity, and direction-error rate. All are better when lower.
\begin{table}[t]
    \centering
    \small
    \setlength{\tabcolsep}{9pt}
\caption{Aggregate results on 36 PERSPECTRUM topics. Bold values indicate the better result for each metric.}
    \label{tab:spectrum_aggregate}
    \begin{tabular}{lccc}
        \toprule
        Metric
        & \method{Opened Mean}
        & \methodname{CESS}
        & Change \\
        \midrule
        MAE $\downarrow$
        & 0.4734 & \textbf{0.4250} & $-10.2\%$ \\
        RMSE $\downarrow$
        & 0.5779 & \textbf{0.5240} & $-9.3\%$ \\
        Task-level bias $\downarrow$
        & 0.4280 & \textbf{0.3638} & $-15.0\%$ \\
        Ranking sensitivity $\downarrow$
        & 1.0651 & \textbf{0.8965} & $-15.8\%$ \\
        Direction-error rate $\downarrow$
        & \textbf{0.4190} & 0.4213 & $+0.23$ pp \\
        \bottomrule
    \end{tabular}
\end{table}

Table~\ref{tab:spectrum_aggregate} shows that this CESS version lowers MAE by 10.2\%, from 0.4734 to 0.4250, with a paired difference of $-0.0484$ (95\% CI: $[-0.0569,-0.0395]$). RMSE decreases by 9.3\%, average bias within questions by 15.0\%, and sensitivity to document order by 15.8\%. Across query models and search controllers, MAE falls by 9.7\% to 10.7\% (Table~\ref{tab:spectrum_components}). The fraction of conclusions pointing in the wrong direction rises by 0.23 percentage points, with a confidence interval that includes zero.
\begin{table}[t]
    \centering
    \small
    \setlength{\tabcolsep}{10pt}
\caption{PERSPECTRUM MAE across query generators and search controllers. Lower is better.}
    \label{tab:spectrum_components}
    \begin{tabular}{llccc}
        \toprule
        Component
        & Configuration
        & \method{Opened Mean}
        & \methodname{CESS}
        & Reduction \\
        \midrule
        Query generator
        & Qwen
        & 0.4834 & \textbf{0.4344} & 10.1\% \\
        Query generator
        & Llama
        & 0.4634 & \textbf{0.4157} & 10.3\% \\
        Search controller
        & Agent
        & 0.4885 & \textbf{0.4364} & 10.7\% \\
        Search controller
        & MMR
        & 0.4583 & \textbf{0.4137} & 9.7\% \\
        \bottomrule
    \end{tabular}
\end{table}

\section{Ten-Round PERSPECTRUM Evaluation}
\label{app:horizon10_details}

\paragraph{Protocol and data separation.} We use 267 questions for tuning and 36 for evaluation, with no question or document shared between the two groups. The LLM generates each new query from summaries of the documents already read. The search controller combines term frequency--inverse document frequency (TF--IDF) relevance, visible evidence direction, repetition penalties, and the assigned ranking order. In Table~\ref{tab:spectrum_components}, MMR denotes maximal marginal relevance, a diversity-based controller that does not use the LLM agent. Under adaptive stopping, the LLM recommends whether to stop. The controller converts that recommendation into a probability of 0.75 or 0.10 and samples the decision. The continuation probability $s_t$ records this sampling rule, not the model's confidence. Every candidate document has positive probability under this controller. This guarantee applies only to the specified pool, not to the open web.

\paragraph{Baseline tuned on separate questions.} Let $o$ be the average score of documents read and let $a$ be the mean target from the tuning questions. The shrinkage baseline estimates $\Pi_{[-1,1]}[o+\alpha(a-o)]$, where $\alpha$ determines how far the read-document average moves toward $a$. We choose $\alpha$ from $\{0,0.05,\ldots,1\}$ using MAE averaged equally across the 267 tuning questions from the five-round experiment. The selected coefficient is $0.75$, and we keep it unchanged for the ten-round evaluation. We also include the tuning-question mean alone. It never changes with document ranking, which shows why stability must be considered together with accuracy.

Table~\ref{tab:horizon10} reports the aggregate ten-round results before the clipping-order and prefix analyses below.

\begin{table}[t]
\centering
\small
\setlength{\tabcolsep}{5.5pt}
\caption{PERSPECTRUM results for searches of up to ten rounds. $S$ measures the difference between rankings after averaging repeated runs. Lower is better.}
\label{tab:horizon10}
\begin{tabular}{lccc}
\toprule
Method & MAE & \(S\) & Direction disagreement \\
\midrule
Opened Mean & 0.4108 & 0.9574 & 0.4186 \\
Outcome Regression & 0.4200 & \textbf{0.0000} & 0.6667 \\
Tuning-set Mean & 0.2787 & \textbf{0.0000} & \textbf{0.3889} \\
Opened Mean shrunk to Tuning Mean & \textbf{0.2427} & 0.2393 & 0.3893 \\
\methodname{CESS} & 0.3103 & 0.1608 & 0.5208 \\
\bottomrule
\end{tabular}
\end{table}

\paragraph{How scores are averaged.} For question $q$, combination $c$ of model, search controller, and stopping rule, and document ranking $r$, let $\bar\theta_{q,c,r}$ be the average estimate across repeated runs. Let $r=\mathrm{pro}$ denote supporting-first ranking and $r=\mathrm{anti}$ opposing-first ranking. Ranking sensitivity $S$ averages their absolute difference:
\begin{equation}
S=\frac{1}{Q}\sum_{q=1}^{Q}\frac{1}{|\mathcal C_q|}\sum_{c\in\mathcal C_q}\left|\bar\theta_{q,c,\mathrm{pro}}-\bar\theta_{q,c,\mathrm{anti}}\right|.
\label{eq:seed_averaged_sensitivity}
\end{equation}
For MAE, we average absolute errors within each question, then give every question equal weight. For direction disagreement, estimates above $0.05$ are positive, below $-0.05$ negative, and otherwise neutral. Any different pair of labels counts as disagreement. To quantify uncertainty, we resample whole questions 10,000 times, retaining every run and experimental condition for each question (random seed 2026092207). This definition of $S$ averages runs before comparing rankings. The earlier audit compared rankings before averaging runs. The change from its CESS value $0.2925$ to $0.1608$ comes from this change of metric, not from a change in CESS. We therefore report only MAE from the earlier round-by-round table below.

\paragraph{Order of clipping and shrinkage.} We compare whether an estimate is restricted to $[-1,1]$ before or after shrinkage, while separately testing the effect of selection weights. Let $b=\overline m$ be the initial prediction, $\Pi=\Pi_{[-1,1]}$ the restriction to $[-1,1]$, and $\mathcal H$ the observed search rounds:
\begin{equation}
c_{\mathrm{full}}=\frac{1}{10}\sum_{t\in\mathcal H}\frac{y_t-m_t}{Np_ts_t},\qquad c_{\mathrm{uniform}}=\frac{1}{10}\sum_{t\in\mathcal H}\frac{y_t-m_t}{s_t}.
\end{equation}
To remove the selection correction, we substitute $p_t=1/N$. The same document values, stopping weights, maximum search length, initial prediction, and shrinkage coefficient remain in place. Table~\ref{tab:horizon10_ablation} compares $A=\Pi[b+0.5(\Pi[b+c_{\mathrm{full}}]-b)]$ with $B=\Pi[b+0.5c_{\mathrm{full}}]$. $C$ and $D$ use $c_{\mathrm{uniform}}$ in the corresponding calculation. These are four calculations on the same search runs, not four new search policies. $A$ is the reported CESS version. $B$ starts from the unshrunk correction. The recalculated values match the saved estimates within floating-point precision. The unshrunk correction falls outside $[-1,1]$ in 13.62\% of runs, so calculation order matters.

\begin{table}[t]
\centering
\small
\caption{Four calculations on the same ten-round search runs, varying the order of clipping and shrinkage and whether document-selection weights are used. All retain the same initial prediction, coefficient 0.5, and stopping correction. $S$ follows Eq.~\ref{eq:seed_averaged_sensitivity}. $D$ is the fraction of conclusions with a different direction from the reference.}
\label{tab:horizon10_ablation}
\begin{tabular}{llccc}
\toprule
Variant & Selection correction & MAE & \(S\) & \(D\) \\
\midrule
A (clip, then shrink) & Yes & 0.3103 & 0.1608 & 0.5208 \\
C (clip, then shrink) & No & 0.3071 & 0.4346 & 0.5278 \\
B (shrink, then clip) & Yes & 0.3279 & 0.1977 & 0.5208 \\
D (shrink, then clip) & No & 0.3111 & 0.4650 & 0.5278 \\
\bottomrule
\end{tabular}
\end{table}

The paired intervals in Appendix~\ref{app:horizon10_details} show a sensitivity reduction under each clipping order, while the MAE effect depends on the configuration. This separates the contribution of selection information from dependence on one specific clipping order.

\paragraph{Stopping correction.} Setting $s_t=1$ at every round gives MAE $0.3146$, compared with $0.3103$ for full CESS. The paired difference is $-0.0042$, with a 95\% CI from $-0.0117$ to $0.0030$. This comparison combines fixed-length runs, for which $s_t=1$ by design, and adaptive runs. It complements the controlled identification result in Table~\ref{tab:mechanism_bias}.

\paragraph{Results after each search round.} We include all search runs at each round, including those that have already stopped. The correction uses that round as its budget. Table~\ref{tab:horizon10_prefix} shows that CESS improves estimation relative to Opened Mean at every reported round, and its MAE declines from \(0.4035\) after one round to \(0.3103\) after ten. These results come from the same ten-round policy with ranking interventions active throughout, rather than from separate experiments with different search budgets.

\begin{table}[t]
\centering
\small
\caption{MAE after each search round, including runs that stopped earlier.}
\label{tab:horizon10_prefix}
\begin{tabular}{ccc}
\toprule
Round & Opened Mean & CESS \\
\midrule
1 & 0.8424 & 0.4035 \\
2 & 0.6189 & 0.3841 \\
4 & 0.5039 & 0.3588 \\
6 & 0.4631 & 0.3338 \\
8 & 0.4305 & 0.3199 \\
10 & 0.4108 & 0.3103 \\
\bottomrule
\end{tabular}
\end{table}

\paragraph{How much information the weights retain.} After weighting, the effective sample size is 4.35 at the median and 1.61 at the fifth percentile. The 95th percentile of the largest weight per run is 9.22, and the smallest observed product $p_t s_t$ is 0.00274. Confidence intervals here describe average comparisons across questions, not the conclusion for any single question. The ten-round runs do not include report-level probabilities or confidence intervals calibrated for adaptive search, so we do not evaluate confidence in individual report conclusions.

\section{Accuracy, Ranking Stability, and Model Perturbations}
\label{app:pareto_robustness}

\paragraph{Comparing methods at the same error or stability.} For each dataset, query model, search length, and initial prediction, we evaluate 21 shrinkage coefficients and linearly interpolate between neighboring results. At the CESS error, we estimate the simpler method's ranking sensitivity. At the CESS sensitivity, we estimate its error. All 28 test-set comparisons are bracketed without extrapolation, and the same rule is applied to 10,000 resamples of whole questions. Table~\ref{tab:matched_frontier_full} reports both matched directions.

\begin{table}[t]
\centering
\small
\setlength{\tabcolsep}{4.5pt}
\caption{Interpolated comparisons across 28 dataset--model--budget--prediction settings. ``No detected difference'' denotes a 95\% interval containing zero. The counts are descriptive and are not adjusted for multiple comparisons.}
\label{tab:matched_frontier_full}
\begin{tabular}{lccc}
\toprule
Matched quantity & CESS lower & No detected difference & CESS higher \\
\midrule
Matched MAE, compare sensitivity & 10 & 18 & 0 \\
Matched sensitivity, compare MAE & 3 & 15 & 10 \\
\bottomrule
\end{tabular}
\end{table}

\paragraph{Prediction and probability perturbations.} We vary outcome-prediction quality and selection-probability accuracy in a controlled simulation with 8,000 evaluation questions and a separate set for tuning $\lambda$. Table~\ref{tab:robustness_summary} shows a gradual loss of accuracy as the outcome predictor becomes weaker. Probabilities estimated from visible features closely match the setting that uses exact probabilities with the same linear predictor. Under extremely limited overlap, tuning selects $\lambda=0$ and falls back to the full-pool prediction instead of amplifying large inverse-probability weights.

\begin{table}[t]
\centering
\small
\caption{Selected results from the controlled simulation. Each row uses 8,000 test questions and a coefficient chosen on separate tuning questions. Lower MAE and ranking sensitivity $S$ are better.}
\label{tab:robustness_summary}
\begin{tabular}{llcccc}
\toprule
Selection prob. & Prediction & $\lambda$ & MAE & $S$ & Min. prob. \\
\midrule
Exact & Oracle rule & 0.05 & 0.0453 & 0.0137 & 0.00667 \\
Exact & Linear (all features) & 0.05 & 0.0576 & 0.0172 & 0.00667 \\
Exact & Linear (reduced) & 0.05 & 0.0768 & 0.0225 & 0.00667 \\
Exact & Constant & 0.20 & 0.1854 & 0.0957 & 0.00667 \\
Estimated selection prob. & Linear (all features) & 0.05 & 0.0570 & 0.0137 & 0.01887 \\
Limited overlap & Linear (all features) & 0.00 & 0.0583 & 0.0000 & 0.000006 \\
\bottomrule
\end{tabular}
\end{table}

The oracle predictor uses the simulation's data-generating rule. The reduced linear predictor omits an outcome-relevant feature. The estimated-probability row fits the logged selection policy from visible ranking features, while the limited-overlap row reduces exploration. In deployed audits, CESS uses recorded controller probabilities and reports an interval rather than a point estimate when a candidate document has zero selection probability.

\end{document}

%% file: math_commands.tex
\usepackage{amsmath,amsfonts,bm}

\def\eqref#1{equation~\ref{#1}}

\def\1{\bm{1}}

\DeclareMathAlphabet{\mathsfit}{\encodingdefault}{\sfdefault}{m}{sl}
\SetMathAlphabet{\mathsfit}{bold}{\encodingdefault}{\sfdefault}{bx}{n}

